\documentclass{article}
\usepackage{pgfplotstable}
\usepackage{booktabs}  

\usepackage{siunitx}       

\usepackage{spconf,amsmath,amssymb,amsthm,graphicx}
\usepackage{algorithm}
\usepackage{algpseudocode}
\algrenewcommand\algorithmicrequire{\textbf{Input:}}
\algrenewcommand\algorithmicensure{\textbf{Output:}}
\usepackage{booktabs}
\usepackage{pgfplots,arydshln}
\pgfplotsset{compat=1.18}     
\usepgfplotslibrary{fillbetween}
\usetikzlibrary{intersections}
\usepackage{changepage}

\usepackage{url}
\usepackage{subcaption} 

\usepackage{soul}         
\usepackage{arydshln}

\usepackage{amsmath,amssymb,amsfonts}
\usepackage{algorithm}      
\usepackage{algpseudocode}  
\usepackage{xcolor}

\usepackage[normalem]{ulem}

\input{mysymbol.sty}

\title{BUILD with Precision: Bottom-Up Inference of Linear DAGs}

\name{Hamed Ajorlou$^{1}$, Samuel Rey$^{2}$, Gonzalo Mateos$^{1}$, Geert Leus$^{3}$, and Antonio G. Marques$^{2}$\thanks{Work in this paper was supported by the NSF award ECCS-2231036, the Spanish AEI (10.13039/501100011033) grant PID2022-136887NB-I00, and the Community of Madrid via the Ellis Madrid Unit and grants URJC/CAM F1180 and TEC-2024/COM-89.}}

\address{
$^{1}$ University of Rochester, Rochester, NY, USA \\
$^{2}$ Universidad Rey Juan Carlos, Madrid, Spain \\
$^{3}$ Delft University of Technology, Delft, The Netherlands
}

\begin{document}
\ninept
\maketitle
\begin{abstract}
Learning the structure of directed acyclic graphs (DAGs) from observational data is a central problem in causal discovery, statistical signal processing, and machine learning. Under a linear Gaussian structural equation model (SEM) with equal noise variances, the problem is identifiable and we show that the ensemble precision matrix of the observations exhibits a distinctive structure that facilitates DAG recovery. Exploiting this property, we propose BUILD (Bottom-Up Inference of Linear DAGs), a deterministic stepwise algorithm that identifies leaf nodes and their parents, then prunes the leaves by removing incident edges to proceed to the next step, exactly reconstructing the DAG from the true precision matrix. In practice, precision matrices must be estimated from finite data, and ill-conditioning may lead to error accumulation across BUILD steps. As a mitigation strategy, we periodically re-estimate the precision matrix (with less variables as leaves are pruned), trading off runtime for enhanced robustness. Reproducible results on challenging synthetic benchmarks demonstrate that BUILD compares favorably to state-of-the-art DAG learning algorithms, while offering an explicit handle on complexity.
\end{abstract}
\begin{keywords}
DAG structure learning; graphical model; precision matrix; topology inference; causal discovery
\end{keywords}
\section{Introduction}
\label{sec:intro}

Recovering the structure of directed acyclic graphs (DAGs) from observational data is a fundamental problem in signal processing, statistics, and machine learning, with applications in causal discovery and graphical modeling~\cite{ortega2018,peters2017,xia2021}. A common modeling framework is the linear structural equation model (SEM), where each observed variable depends linearly on its parents in the latent DAG together with exogenous noise~\cite{shimizu2006,peters2017}. This paper addresses the structure identification problem in the specific setting of linear Gaussian SEMs with known, equal error variances. In this setting, the DAG is known to be identifiable from observational data alone~\cite{peters2014}, motivating the development of several computational methods~\cite{chen2019,gao2020,ghoshal2017a,ghoshal2017b}.\vspace{2pt}

\noindent \textbf{Related work.} The problem of learning DAGs has a long history; see e.g.,~\cite[Ch. 7.2]{peters2014} for a thorough treatment. Early approaches relied on purely discrete or combinatorial search methods in the space of DAGs. Noteworthy representatives include the Greedy Equivalence Search (GES)~\cite{chickering2002,ramsey2017} and constraint-based methods such as the PC algorithm~\cite{spirtes2000,buhlmann2014cam}, though these approaches scale poorly due to the super-exponential growth (on the number of nodes) of the search space~\cite{chickering1996,chickering2004}. 
A recent line of work develops continuous relaxations of score-based methods, where acyclicity is enforced through smooth functions amenable to gradient-based optimization. NOTEARS pioneered the trace-exponential acyclicity characterization~\cite{zheng2018notears}, with later refinements via polynomial and log-determinant functions whose zeroth level set is the space of DAGs~\cite{yu2019,wei2020,bello2022dagma,samu2025icassp}. 
Subsequent contributions, including GOLEM~\cite{ng2020golem}, DAGMA~\cite{bello2022dagma}, and CoLiDE~\cite{saboksayr2024colide}, advanced this direction by optimizing least-squares or likelihood-based objectives augmented by acyclicity penalties, with CoLiDE further incorporating noise variance estimation for adaptivity. 
These methods scale well to large number of variables, but can face optimization challenges as they attempt to solve non-convex, equality-constrained problems~\cite{wei2020}. When edge weights are non-negative, then convex formulations are possible~\cite{samu2025icassp}. Here instead, we deal with general (non-sign constrained) weights and sidestep optimization issues altogether. Order-based methods recover DAGs by first learning a node ordering and then estimating edges. Recent state-of-the-art representatives include~\cite{gao2022optimal,daskalakis2025}, but these algorithms only estimate the DAG structure (neglecting edge weights, unlike ours), albeit with strong theoretical support recovery guarantees.

More germane to our approach in this paper, another way to tackle DAG learning is to restrict the search space to a superstructure, namely an undirected skeleton encompassing a restricted class of possible DAGs. For Gaussian data, the precision (i.e., inverse covariance) matrix of the observations encodes conditional independencies and has well-documented links to the so-termed moralized graph associated to the DAG~\cite{peters2017,peters2014,lauritzen1996graphical}; see also~\cite{loh2014} for an influential work exploring general non-Gaussian settings. 
This creates a strong incentive to leverage high-quality precision matrix estimators, such as graphical Lasso~\cite{friedman2008glasso}, as a precursor to DAG structure inference. This is the path we follow here, but cautiously, since we show that for linear SEMs the resulting precision matrix can be severely ill-conditioned. 
Interestingly, GreedyPrune avoids restrictive condition number assumptions and provides fixed polynomial-time guarantees for recovering the precision matrix in certain classes of Gaussian graphical models~\cite{kelner2020greedy}. Motivated by these insights on the role of precision matrices to unveil latent DAG structure, we next present our own approach that builds directly on this valuable connection.\vspace{2pt}


\noindent \textbf{Contributions.} We propose BUILD (Bottom-Up Inference of Linear DAGs), 
a deterministic algorithm that iteratively identifies leaf nodes and their parents, then prunes the leaves by removing incident edges to proceed to the next step, exactly reconstructing the DAG from the ensemble precision matrix of a linear Gaussian SEM model with equal noise variances. The algorithm is motivated by the favorable structure we reveal in the population-level precision matrix (see Lemma 1 and Corollary 1). In practice, we only have data to work with and hence an imperfect precision matrix estimate.  A key feature of BUILD is an error mitigation strategy that balances accuracy and runtime by periodically re-estimating the precision matrix, as leaves are pruned and the problem dimensionality decreases. 
Through extensive synthetic experiments, \textsc{BUILD} demonstrates strong performance across multiple evaluation metrics, achieving accurate edge detection and weight estimation while maintaining competitive runtime. Moreover, its performance scales favorably as the number of nodes increases. To ensure transparency and reproducibility, we publicly share the implementation of BUILD alongside the paper.



\section{Preliminaries and Problem Statement}
\label{sec:problem}

We introduce the required background on DAGs and linear SEMs needed to formally state the DAG topology inference problem.\vspace{2pt}

\noindent {\bf Directed acyclic graphs.} Let $\ccalD = (\ccalV, \ccalE)$ denote a DAG, where $\ccalV=\{1,\ldots,N\}$ is the set nodes and $\ccalE \subseteq \ccalV \times \ccalV$ is the set of directed edges. The convention we adopt is that $(i,j)\in\ccalE$ means that there is an arc $j\to i$. Because $\ccalD$ contains no directed cycles or self loops, $\ccalV$ admits a (generally non-unique) topological ordering whereby $(i,j) \in \ccalE$ implies that node $j$ precedes node $i$. Thus, every DAG defines a unique partial order over the set $\ccalV$, where $j < i$ if node $j$ is a \emph{predecessor} of $i$, meaning there exists a directed path from $j$ to $i$. The weighted adjacency matrix $\bbA \in \reals^{N \times N}$ encodes the connectivity structure of $\ccalD$, with $A_{ij} \neq 0$ if and only if $(i,j) \in \ccalE$. When $\ccalV$ is ordered topologically, $\bbA$ is strictly lower triangular. \vspace{2pt}

\noindent {\bf Linear structural equation models.} Let us assume $\ccalD$ captures conditional independencies among the variables in the random vector $\bbx = [ x_1, \ldots, x_N ]^{\top}\in \reals^N$. If the joint distribution $\mathbb{P}(\bbx)$ satisfies a Markov property over $\ccalD$, it implies that each random variable $x_i$ is solely dependent on its parents $\textrm{PA}_i = \{j \in \ccalV : A_{ij} \neq 0\}$~\cite{peters2017}. 
This work focuses on \emph{linear} SEMs to generate such a probability distribution, where the relationship between each random variable and its parents is expressed as $\smash{x_i = \sum_{j\in \textrm{PA}_i}A_{ij}x_{j} + z_i}$, $\forall i\in\ccalV$, where $\bbz = [z_1, \ldots, z_N]^{\top}$ is a zero-mean Gaussian vector of mutually independent, exogenous noises with known variance $\sigma^2$. For a dataset $\bbX\in \reals^{N \times M}$ consisting of $M$ i.i.d. samples drawn from $\mathbb{P}(\bbx)$, the linear SEM can be expressed in matrix form as $\bbX = \bbA\bbX + \bbZ$.
\vspace{2pt}

\noindent {\bf Problem statement.} Given the data matrix $\bbX$ generated by a linear SEM, the goal is to recover the underlying DAG $\ccalD$ by estimating its adjacency matrix $\bbA$. Under the assumption of Gaussian $\bbz$ with equal variances, then $\ccalD$ is identifiable from the joint distribution of $\bbx$ alone (i.e., from observational data); see e.g.,~\cite[Prop. 7.5]{peters2017}. 

\section{BUILD: Bottom-Up Inference of Linear DAGs}~\label{method}
Here, we present the proposed DAG inference algorithm. For a linear Gaussian SEM with known equal noise variances, we first describe how the adjacency matrix $\bbA$ of $\ccalD$ can be reconstructed from the precision matrix $\bbTheta$ of $\bbx$. We follow a bottom-up approach that sequentially identifies and prunes leaf nodes and incident edges from their parents. We then touch upon important implementation details, namely ill-conditioned precision matrix estimation and mitigation of finite-sample induced errors in the DAG reconstruction process.\vspace{2pt}

\noindent\textbf{DAG recovery from the ensemble precision matrix.} Recall the linear SEM $\bbx=\bbA\bbx +\bbz$, where $\bbz\sim\ccalN(\mathbf{0},\sigma^2\bbI_N)$. The signal covariance matrix is given by $\bbSigma:=\E{\bbx\bbx^\top}=\sigma^2(\bbI_N-\bbA)^{-1}(\bbI_N-\bbA^\top)^{-1}$ and, hence, the precision matrix $\bbTheta=\bbSigma^{-1}$ has the form
\begin{align}
\bbTheta  &= \sigma^{-2} \big(\bbI_N - \bbA^\top\big)\big(\bbI_N - \bbA\big)\nonumber\\
&=\sigma^{-2} \left(\bbI_N-\bbA-\bbA^\top+\bbA^\top\bbA\right).\label{eq:eqexp}
\end{align}
The following simple result characterizes the entries of the precision matrix $\bbTheta$, and will be central to our algorithmic approach.
\begin{lemma}[Precision matrix entries]\label{lemma:theta_entries}
\normalfont  Let $\bbA=[\bba_1,\ldots,\bba_N]\in\reals^{N\times N}$ be the adjacency matrix of $\ccalD$ and write $\textrm{supp}(\bba_j)\equiv \textrm{CH}_j:=\{i \in \ccalV : A_{ij} \neq 0\}$, the  children of $j\in\ccalV$. The (scaled by $\sigma^2$) entries of the (symmetric) precision matrix $\bbTheta$ in \eqref{eq:eqexp}  are given by
\begin{equation}\label{eq:Theta_entries}
\sigma^{2}\Theta_{ij}=\left\{\begin{array}{cc}
1+\sum_{k\in\textrm{CH}_i}A_{ki}^2, & i=j\\
-A_{ij}+\sum_{k\in\textrm{CH}_i \cap \textrm{CH}_j}A_{ki}A_{kj},& i>j
\end{array}\right..
\end{equation}
\end{lemma}
\begin{proof}
Follows immediately by noticing $\textrm{diag}(\bbA)=\mathbf{0}_{N}$ because a DAG has no self loops, and from $[\bbA^\top \bbA]_{ij}=\bba_i^\top \bba_j$.
\end{proof}
%
%
The support of $\bbTheta$ corresponds to the so-termed moralized graph of $\ccalD$~\cite{loh2014}, an undirected graph obtained by connecting all nodes within each parent set $\textrm{PA}_i$ and also dropping the directionality of all edges in $\ccalE$. 
Now, if node $i$ is a leaf of $\ccalD$ then $\textrm{CH}_i=\emptyset$ and $\bba_i=\mathbf{0}_N$. This observation along with Lemma \ref{lemma:theta_entries} suggests leaf nodes can be unequivocally identified from the diagonal entries of $\bbTheta$.
\begin{corollary}\label{cor:leafs}
A node $i\in\ccalV$ is a leaf of $\ccalD$ if and only if $\Theta_{ii}=\sigma^{-2}$. All other non-leaf nodes have $\Theta_{jj}>\sigma^{-2}$ and the gap or resolution limit is lower bounded by $\Delta=\sigma^{-2}\min_{i,j\in\ccalV}A_{ij}^2.$
\end{corollary}
%
Moreover, it follows from \eqref{eq:Theta_entries} that we can also identify the parents $\textrm{PA}_i$ of leaf $i$ and the edge weights $A_{ij}\neq 0$ (recall $\sigma^2$ is known). This is because $\textrm{CH}_i=\emptyset$, hence the off-diagonal entries in the $i$-th row (and column) of $\bbTheta$ are scaled (by $-\sigma^{-2}$) copies of the $i$-th row of $\bbA$, whose support is precisely $\textrm{PA}_i$. All in all, if we are given $\bbTheta$ and identify $i\in\ccalV$ as a leaf (cf. Corolloray \ref{cor:leafs}), then we can recover the $i$-th row of the adjacency matrix as 
\begin{equation}\label{eq:leaf_adjacency_row}
    A_{ij}=\left\{\begin{array}{cc}
-\sigma^2\Theta_{ij}, & j<i\\
0,& j\geq i
\end{array}\right..
\end{equation}
Of course, we also know the $i$-th column of $\bbA$ is $\bba_i=\mathbf{0}_N$. 
This observation forms the key insight behind the proposed DAG recovery method, at least in the idealized setting where the ensemble precision matrix $\bbTheta$ is available. The idea is to iteratively identify leaf nodes using Corollary \ref{cor:leafs}, recover their parents from $\bbTheta$ using \eqref{eq:leaf_adjacency_row}, and then remove the contribution of each leaf (the node itself and its incident edges) using \eqref{eq:Theta_entries}. The resulting matrix describes the conditional dependence relations of the reduced DAG.  Repeating this stepwise bottom-up process enables full recovery of the underlying DAG. 

The two last steps of the aforementioned process can be related to eliminating a leaf node from the data and recomputing the precision matrix over the leftover nodes. To clarify this relation, let us split the precision (and covariance) matrix in 4 blocks (for visual convenience we permute rows and columns so that leaf $i$ is depicted first) and, a block matrix inversion identity yields 
%
\begin{align}\label{eq:prune}
\bbTheta &=
\left[
\begin{array}{c:c}
[\bbTheta]_{ii} & [\bbTheta]_{i, \ccalR} \\ \hdashline
\\
\left[\bbTheta\right]_{\ccalR, i} & \;\;{ [\bbTheta]_{\ccalR, \ccalR}}\;\; \\
\\
\end{array}
\right],\\
\bbTheta_{\ccalR \ccalR}  
   &= [\bbTheta]_{\ccalR, \ccalR}  
     - [\bbTheta]_{ii}^{-1} \, [\bbTheta]_{\ccalR,i}\,[\bbTheta]_{i, \ccalR},\label{e:new_precision}
\end{align}
where $\ccalR=\ccalV\setminus \{i\}$ is the set of remaining nodes, $[\bbTheta]_{ii}=\Theta_{ii}=\sigma^{-2}$ is the scalar diagonal entry corresponding to the selected leaf $i$,  
$[\bbTheta]_{i, \ccalR} \in \reals ^{1 \times|\ccalR|}$ is the row vector with entries $-\sigma^{-2}A_{ij}$, $j\in\ccalR$, 
and $[\bbTheta]_{\ccalR \ccalR} \in \reals^{|\ccalR|\times|\ccalR|}$ is the precision submatrix over the leftover nodes (prior to pruning). After removing $i$, the precision matrix for the next step is 
$\bbTheta_{\ccalR \ccalR}   = [\bbTheta]_{\ccalR, \ccalR}   - [\bbTheta]_{ii}^{-1} \, [\bbTheta]_{\ccalR,i}\,[\bbTheta]_{i, \ccalR}$. If $i$ is indeed a leaf node, the elimination of the $i$-th row and column and the update of the remaining entries via \eqref{e:new_precision} is similar to the process described in the previous paragraph.  

\begin{algorithm}[t]
\caption{\textsc{BUILD}: Bottom-Up Inference of Linear DAGs}
\label{alg:builder}
\begin{algorithmic}[1]
\Require $\hbTheta$, $\bbX$, $\sigma^2$, 
threshold $\bbvarepsilon$, refresh rate $\rho\in[0,1]$
\Ensure $\widehat{\bbA}\in\reals^{N\times N}$
\State $\bbA\gets \bb0_{N\times N}$, $\tau\gets 0$,\: $\bbTheta\gets \hbTheta$,\; $\mathcal{Q}\gets\emptyset$ \Comment{pruned leaves in $\mathcal{Q}$}
\State $\ccalT\gets \{\lfloor t\rho N\rfloor: t=1,2,\ldots,N-1\}
$ \Comment{refresh checkpoints}
\While{$|\mathcal{Q}|<N$}
    \State $\mathcal{R}\gets \{1,\ldots,N\}\setminus\mathcal{Q}$ \Comment{remaining nodes}
  \If{$\tau\in\ccalT$} \Comment{time to re-estimate $\bbTheta_{\ccalR\ccalR}$}
      \State $\hbTheta_{\ccalR\ccalR}\gets 
      \textsc{GreedyPrune}(\bbX_{\mathcal{R}})$ 
      \State $[\bbTheta]_{\ccalR,\ccalR}\gets\hbTheta_{\ccalR\ccalR}$ 
  \EndIf
  \State $i\gets \arg\min_{u\notin\mathcal{Q}}\{\Theta_{uu}\ge \bbvarepsilon\}$ \textbf{ break if none} \Comment{find leaf}
  \State $\mathcal{Q}\gets \mathcal{Q}\cup\{i\}$,\; $\tau\gets \tau+1$
  \State $\bba\gets -\sigma^2\bbTheta_{i,:}$;\; set $a_i\gets 0$;\; $\forall j\neq i,$ set $a_j\gets 0$ if $|a_j|<\bbvarepsilon$ 
  \State $\bbA_{i, :}\gets \bba$ \Comment{connectivity of $i$ to its parents [cf. \eqref{eq:leaf_adjacency_row}]}
  \For{each $j$ with $|a_j|\ge \bbvarepsilon$}
    \State $\bbTheta_{j,:}\gets \bbTheta_{j,:}-\sigma^{-2} a_j\bba$ \Comment{effect of $i$ on $\ccalR$ [cf. \eqref{e:new_precision}]}
  \EndFor
  \State $\bbTheta_{i,:}\gets \bb0$,\; $\bbTheta_{:\!,i}\gets \bb0$ \Comment{prune leaf node $i$}
\EndWhile
\State \Return $\widehat{\bbA}=\bbA$
\end{algorithmic}\label{alg:BUILD}
\end{algorithm}

Building on this core idea, we now introduce our proposed approach, \textsc{BUILD} (Bottom-Up Inference of Linear DAGs), which operates in two phases. Given data $\bbX \in \reals^{N \times M}$, in the first phase we estimate the precision matrix borrowing existing methods from the literature. In the second phase (our main contribution), we apply Algorithm~\ref{alg:BUILD} to recover the underlying DAG structure from $\hbTheta$. \vspace{2pt}

\noindent\textbf{Implementation details}. With access to the ensemble precision matrix $\bbTheta$, \textsc{BUILD} deterministically recovers the adjacency matrix $\bbA$ in $\ccalO(N^2)$ time. In practice, however, the population-level precision matrix is unknown and it must be inferred from observational data $\bbX$. 
%
%
%
This can be particularly challenging for high-dimensional linear SEMs with equal noise variance $\sigma^2$, as the covariance matrix $\bbSigma=\sigma^2(\bbI_N-\bbA)^{-1}(\bbI_N-\bbA^\top)^{-1}$ (and hence $\bbTheta$) will be badly conditioned. Indeed, the Neumann expansion $(\bbI_N-\bbA)^{-1}=\sum_{k=0}^{N-1} \bbA^k$ reveals that variances at downstream nodes in $\ccalD$ accumulate noise contributions from all their ancestors. This noise accumulation effect leads to large disparities in the marginal variances of $x_1,\ldots,x_N$, and thus in the eigenvalues of $\bbSigma$. The result is an ill-conditioned precision matrix, a phenomenon particularly severe in DAGs with long dependency chains, high branching, or heterogeneous edge weights. For instance, in a chain $1\to 2\to\cdots\to N$ with edge weights $A_{i,i+1}=k>1$, the condition number of $\bbTheta$ grows as $k^{2N}$. 
We employ GreedyPrune~\cite{kelner2020greedy} to estimate the precision matrix in the initial phase, which we found to perform best in ill-conditioned settings where e.g., graphical Lasso fails miserably.

We now provide a more detailed algorithmic description of BUILD. We examine how the finite sample-induced errors in $\hbTheta$ affect the estimation of $\bbA$, and describe our mitigation strategies. At each iteration, we restrict attention to the active set of variables $\ccalR$. We designate as a leaf the node $i\in\ccalR$ whose diagonal entry in the precision matrix attains the smallest value above a fixed tolerance $\epsilon$, thereby avoiding spurious numerical artifacts leading to leaf misidentification. We then recover its parent relationships from the corresponding row of the precision matrix ($\bbTheta_{i,:}$ using Matlab notation), discarding entries with negligible magnitude to suppress the effect of estimation noise. The precision matrix is then updated by removing the contribution of this node (leaf pruning via block matrix inversion described earlier), and the process continues on the reduced system, which has the same structure as before. Repeating this procedure until no nodes remain yields $\hbA$. 

As mentioned earlier, inaccurately estimated edges may fail to completely remove the influence of discovered leaf nodes, leaving residual effects in the updated precision matrix. For large DAGs, these residuals propagate further, compounding over iterations and leading to a snowball effect that may result in catastrophic edge detection errors. To mitigate this phenomenon, we introduce a refreshing scheme in which the precision matrix is re-estimated at fixed intervals. This allows BUILD to operate on an updated precision matrix (with fewer variables as leafs are pruned, thus easier to estimate), resetting past accumulated errors. Algorithm \ref{alg:BUILD} considers only the active nodes $\ccalR$ that have not yet been processed, extracts the corresponding rows and columns of the sample covariance matrix, and re-estimates the precision matrix using GreedyPrune as if operating on a newly pruned graph. We find that this scheme markedly mitigates the effects of error accumulation. In the experiments, we test several variants of \textsc{BUILD} with different refresh rates to provide further insight into the performance versus complexity tradeoffs.\vspace{2pt}



\begin{figure*}
    \centering
    \includegraphics[width=0.83\linewidth]{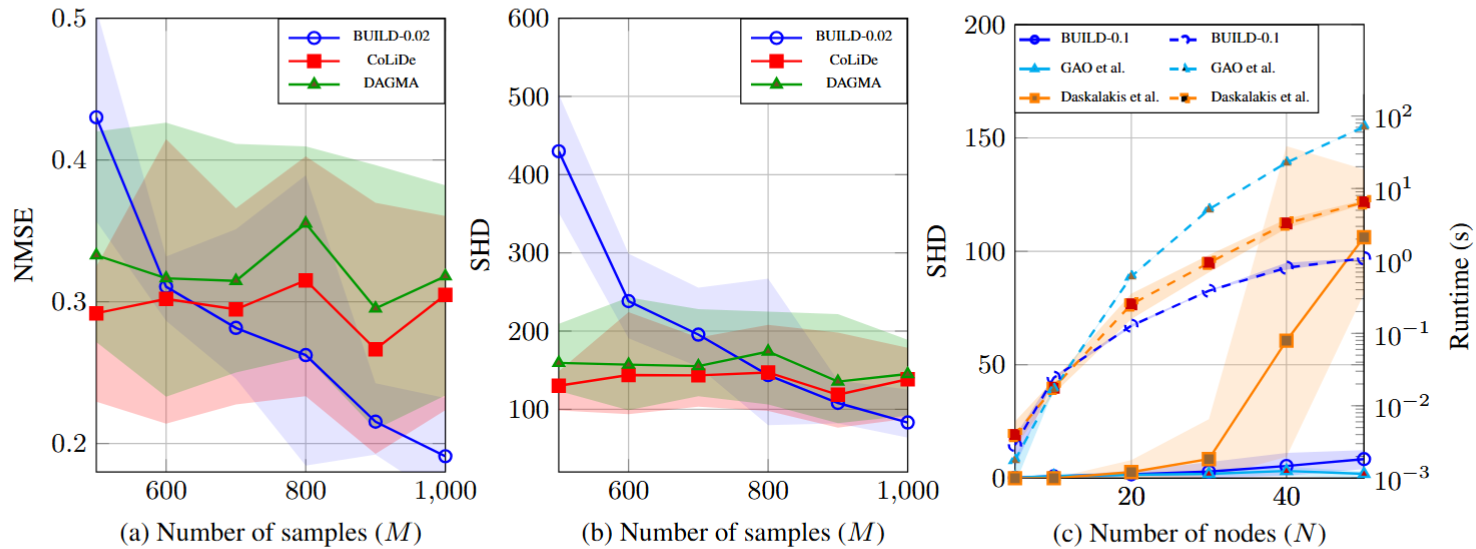}
\caption{Performance of BUILD compared with state-of-the-art baselines. Results are averaged over 10 trials, with shaded regions indicating the 10th and 90th percentiles. 
    (a) NMSE of edge estimation as a function of sample size. 
    (b) SHD as a function of sample size. 
    (c) Comparison with order-based methods: SHD (left) and runtime in log scale (right) versus the number of nodes. \vspace{-.6cm}}
  \label{fig:one-by-three}
\end{figure*}

\noindent\textbf{Closing remarks.}
The main bottleneck of our approach is its sensitivity to errors in $\hbTheta$. The quality of $\hbTheta$ directly
affects \textsc{BUILD}'s ability to recover $\bbA$. 
With more data, the first phase yields a more reliable precision matrix, and downstream DAG estimation improves accordingly. Also, re-estimation of $\hbTheta$ offers added robustness at the price of increased runtime. We examine these tensions empirically in the numerical tests that follow.

\pgfplotstableread[col sep=comma]{data/TABLE_metrics.csv}\results

\newcommand{\meanstd}[3]{%
  \pgfplotstablegetelem{#1}{#2}\of{\results}\edef\themean{\pgfplotsretval}%
  \pgfplotstablegetelem{#1}{#3}\of{\results}\edef\thesd{\pgfplotsretval}%
  \(\themean \pm \thesd\)%
}

\newcommand{\meanstdb}[3]{%
  \pgfplotstablegetelem{#1}{#2}\of{\results}\edef\themean{\pgfplotsretval}%
  \pgfplotstablegetelem{#1}{#3}\of{\results}\edef\thesd{\pgfplotsretval}%
  \(\mathbf{\themean} \pm \mathbf{\thesd}\)%
}

\begin{table}[t]
\centering
\caption{Comparison of continuous optimization methods with variants of BUILD under different refreshing rates. Metrics reported as mean $\pm$ std over 20 trials. Best entries are highlighted in \textbf{bold}.}
\label{tab:main_comparison}
\setlength{\tabcolsep}{3pt}
\renewcommand{\arraystretch}{0.95}
\footnotesize
\resizebox{\columnwidth}{!}{%
\begin{tabular}{lcccc}
\toprule
Baseline & SHD $\downarrow$ & FDR $\downarrow$ & TPR $\uparrow$ & Time (s) $\downarrow$ \\
\midrule
BUILD\texttt{-}0.005 &
  \meanstdb{0}{shd_mean}{shd_std} &
  \meanstdb{0}{fdr_mean}{fdr_std} &
  \meanstdb{0}{tpr_mean}{tpr_std} &
  \meanstd {0}{runtime_mean}{runtime_std} \\
BUILD\texttt{-}0.01  &
  \meanstd{1}{shd_mean}{shd_std} &
  \meanstd{1}{fdr_mean}{fdr_std} &
  \meanstd{1}{tpr_mean}{tpr_std} &
  \meanstd{1}{runtime_mean}{runtime_std} \\
BUILD\texttt{-}0.02  &
  \meanstd{2}{shd_mean}{shd_std} &
  \meanstd{2}{fdr_mean}{fdr_std} &
  \meanstd{2}{tpr_mean}{tpr_std} &
  \meanstd{2}{runtime_mean}{runtime_std} \\
BUILD\texttt{-}0.04  &
  \meanstd{3}{shd_mean}{shd_std} &
  \meanstd{3}{fdr_mean}{fdr_std} &
  \meanstd{3}{tpr_mean}{tpr_std} &
  \meanstd{3}{runtime_mean}{runtime_std} \\
\hline
CoLiDE &
  \meanstd{4}{shd_mean}{shd_std} &
  \meanstd{4}{fdr_mean}{fdr_std} &
  \meanstd{4}{tpr_mean}{tpr_std} &
  \meanstd{4}{runtime_mean}{runtime_std} \\
DAGMA                &
  \meanstd{5}{shd_mean}{shd_std} &
  \meanstd{5}{fdr_mean}{fdr_std} &
  \meanstd{5}{tpr_mean}{tpr_std} &
  \meanstdb{5}{runtime_mean}{runtime_std} \\
\bottomrule
\end{tabular}%
}
\end{table}

\section{Numerical Tests}
\label{sec:Experiments}
We evaluate the performance of BUILD\footnote{Code at \url{https://github.com/hamedajorlou/BUILD}}  by benchmarking it against recent state-of-the-art DAG structure learning algorithms, including CoLiDE~\cite{saboksayr2024colide}, DAGMA~\cite{bello2022dagma}, Gao et al.~\cite{gao2022optimal}, and Daskalakis et al.~\cite{daskalakis2025}. Unless otherwise stated, the synthetic data are generated according to a linear SEM with homoscedastic Gaussian noise with $\sigma^2=1$. For the experiments, we consider Erdős–Rényi DAGs with $N=200$ nodes and expected degree $d=4$. The sample size is $M=1{,}000$ observations. Edge weights $A_{ij}$ are drawn uniformly from the range $(-2,-0.5)\cup(0.5,2)$. This is a standard and challenging setting in the DAG estimation literature~\cite{saboksayr2024colide,bello2022dagma}. 
The task is to recover 
the ground-truth DAG, and performance is evaluated in terms of structural Hamming distance (SHD), true and false positive rates (TPR and FDR), and runtime. Comparative results are reported in Table~\ref{tab:main_comparison}. Mean $\pm$ standard deviation over 20 trials is reported.
\vspace{2pt}

\noindent\textbf{Discussion and findings.} The variants of BUILD reported in Table~\ref{tab:main_comparison} correspond to different precision matrix refresh rates. For instance, BUILD-$0.1$ means $\bbTheta$ is re-estimated whenever 10\% of the nodes have been pruned during the bottom-up DAG construction process.

In principle, BUILD can exactly recover the true underlying DAG if provided with the ensemble precision matrix. However, in practice the precision matrix is only available through empirical estimation, which inevitably introduces errors. As BUILD proceeds by sequentially removing the influence of leaf nodes, these estimation errors propagate and accumulate, ultimately leading to spurious edge detections.
The key insight is that BUILD’s accuracy is highly sensitive to precision matrix estimation error. Refreshing the precision matrix mitigates this issue by discarding the accumulated error and recalibrating the algorithm with a better estimate, thereby improving edge recovery. This effect is most evident in the FDR
reported in Table~\ref{tab:main_comparison}. As the refresh rate increases, the FDR consistently decreases. Naturally, each refresh incurs additional computational cost, creating a trade-off between accuracy and runtime. As an extreme case for a DAG with $N=200$ nodes, we run BUILD-$0.005$, which refreshes the precision matrix after processing each node. While admittedly computationally expensive, this variant achieves substantially higher accuracy than competing continuous optimization methods, highlighting the merits of error correction through re-estimation of $\hbTheta$.\vspace{2pt}

\noindent\textbf{Edge weight estimation.} To assess edge-weight estimation performance not captured by the metrics in Table~\ref{tab:main_comparison}, we compute the Normalized MSE (NMSE) between the estimated adjacency matrix $\widehat{\bbA}$ and the ground truth $\bbA^\star$, defined as $\mathrm{NMSE} = \|\widehat{\bbA} - \bbA^\star\|_F^2 / \|\bbA^\star\|_F^2$, across different sample sizes; see Fig.~\ref{fig:one-by-three}(a). 
In BUILD-$0.02$, the precision matrix is refreshed every $2\%$ of the nodes processed, which for $N=200$ nodes amounts to every $4$ nodes. As $M$ increases, the precision matrix estimation error decreases, leading to more accurate edge recovery in the second phase of the algorithm. This improved accuracy allows BUILD-$0.02$ to outperform the other two baselines, at comparable complexity (see Table~\ref{tab:main_comparison}).\vspace{2pt}


\noindent\textbf{Structure recovery.} In many applications, the primary objective is to recover the correct support of $\bbA$. With this in mind, in Fig.~\ref{fig:one-by-three}(b), we report the SHD of the estimated DAGs as a function of the number of samples $M$. As the sample size grows, the error in precision matrix estimation decreases, which leads to more reliable support recovery. In this setting, BUILD-$0.02$ achieves the lowest SHD and outperforms the other models once $800$ samples are available.\vspace{2pt}

\noindent\textbf{Comparison with order-based methods.} The previous experiments compared BUILD against continuous optimization methods. We now turn our attention to a different line of order-based approaches with promising performance, namely Gao et al.~\cite{gao2022optimal} and Daskalakis et al.~\cite{daskalakis2025}. Unlike optimization-based methods, these algorithms focus on support recovery rather than estimating edge weights, which is an inherent limitation. For this reason, we present their results separately in Fig.~\ref{fig:one-by-three}(c), where the left panel shows SHD and the right panel reports runtime (in log-scale). Dashed lines correspond to the runtime of each baseline. 
We observe that Gao et al.~\cite{gao2022optimal} achieves consistently low error as the number of nodes grows, but its runtime grows rapidly, making it impractical for larger networks. By contrast, Daskalakis et al.~\cite{daskalakis2025} is more efficient but suffers from high error in this edge weight regime, even though it nearly recovers the ground truth when weights are restricted to $(-1,-0.5)\cup(0.5,1)$. In terms of SHD, BUILD-0.1 matches the accuracy of Gao et al.~\cite{gao2022optimal} while remaining scalable, and it is consistently faster than both baselines. This makes BUILD a viable choice for large-scale DAG learning, where the alternatives either become computationally prohibitive or incur significant error.


\section{Conclusions, Limitations, and Future Work}\label{sec:conclusions}

We introduced \textsc{BUILD}, a deterministic algorithm for DAG structure identification using the observations' precision matrix. BUILD iteratively identifies and prunes leave nodes, reconstructing the DAG in a bottom-up fashion. To keep finite sample-induced error propagation in check, optional re-estimation of the precision matrix is considered. On synthetic linear Gaussian SEM data, BUILD achieved lower SHD and higher TPR than representative baselines while remaining computationally competitive. A limitation of the algorithm is its sensitivity to the accuracy of the estimated precision matrix, especially given that the problems tend to be badly conditioned. 
Moreover, refreshing the weights discards many previously computed entries of the precision matrix that are already contaminated by error, which introduces inefficiency. These limitations point toward natural directions for future research, including concomitant estimation of the noise variance and DAG structure, formalizing sample-complexity guarantees while neglecting the refresh mechanism for analytical tractability, designing adaptive refresh schedules to balance accuracy and efficiency, and extending the approach to non-Gaussian and nonlinear SEMs. Exploring structure estimation in the presence of interventional data is also of interest.


\bibliographystyle{IEEEbib}
\bibliography{myIEEEabrv,strings,refs}

\end{document}